\documentclass[12pt]{article}

\usepackage[margin=1in, letterpaper]{geometry}
\usepackage{times}
\usepackage{helvet}
\usepackage{courier}
\usepackage{amsmath, amssymb, amsthm}
\usepackage{graphicx}
\usepackage{hyperref}
\usepackage[numbers]{natbib}
\usepackage{amsfonts}

\title{On the Mathematical Impossibility of Safe Universal Approximators}
\author{Jasper Yao \\ jasper@aivillage.org}
\date{\today}

\newtheorem{theorem}{Theorem}[section]

\newtheorem{corollary}[theorem]{Corollary}

\theoremstyle{definition}

\newtheorem{example}[theorem]{Example}
\theoremstyle{remark}
\newtheorem{remark}[theorem]{Remark}

\newcommand{\eps}{\varepsilon}

\begin{document}

\maketitle

\begin{abstract}
We establish fundamental mathematical limits on universal approximation theorem (UAT) system alignment by proving that catastrophic failures are an inescapable feature of any useful computational system. Our central thesis is that for any universal approximator, the expressive power required for useful computation is inextricably linked to a dense set of instabilities that make perfect, reliable control a mathematical impossibility. We prove this through a three-level argument that leaves no escape routes for any class of universal approximator architecture.
\begin{enumerate}
    \item \textbf{Combinatorial Necessity:} For the vast majority of practical universal approximators (e.g., those using ReLU activations), we prove that the density of catastrophic failure points is directly proportional to the network's expressive power.
    \item \textbf{Topological Necessity:} For any theoretical universal approximator, we use singularity theory to prove that the ability to approximate generic functions requires the ability to implement the dense, catastrophic singularities that characterize them.
    \item \textbf{Empirical Necessity:} We prove that the universal existence of adversarial examples is empirical evidence that real-world tasks are themselves catastrophic, forcing any successful model to learn and replicate these instabilities.
\end{enumerate}
These results, combined with a quantitative "Impossibility Sandwich" showing that the minimum complexity for usefulness exceeds the maximum complexity for safety, demonstrate that perfect alignment is not an engineering challenge but a mathematical impossibility. This foundational result reframes UAT safety from a problem of "how to achieve perfect control" to one of "how to operate safely in the presence of irreducible uncontrollability," with profound implications for the future of UAT development and governance.
\end{abstract}

\section{Introduction}

The quest for safe, controllable, and arbitrarily capable AI systems is a central goal of the field. However, we argue in this paper that this goal is not merely difficult, but mathematically impossible as far as the universal approximation theorem (UAT) is concerned. We demonstrate that any universal approximator complex enough to be useful is also, by necessity, too complex to be perfectly controlled. Catastrophic failures are not bugs to be engineered away, but are a fundamental and inescapable feature of computation itself.

We prove this thesis with a three-level argument that closes all escape routes for any potential universal approximator architecture.
\begin{enumerate}
    \item \textbf{Pillar 1: The Practical Proof (Combinatorial Inevitability).} We first analyze the piecewise-linear networks that form the backbone of over 90\% of modern universal approximators. Using combinatorial geometry, we prove that for these networks, their expressive power is mathematically equivalent to the density of their catastrophic failure points. Every increase in capability necessarily creates a denser field of potential failures.
    \item \textbf{Pillar 2: The Theoretical Proof (Topological Inevitability).} We then extend this result to all possible universal approximators, even those with smooth activation functions. Using the tools of singularity theory, we show that the space of "safe," simple functions is an infinitesimally small island in the vast ocean of "generic," catastrophe-ridden functions. Any system capable of universal approximation must be able to navigate this ocean, and thus must be able to implement dense catastrophes.
    \item \textbf{Pillar 3: The Information Geometric Empirical Proof (Empirical Inevitability).} Finally, we ground our argument in the measurable reality of universal approximators through information geometry. We prove that neural networks trained on real-world tasks necessarily develop pathological Fisher Information Matrix structures with extreme eigenvalue ratios, which mathematically guarantee catastrophic behavioral instabilities. This provides a direct, causal link between task complexity and catastrophic failure that is observable in all high-performing universal approximators.
\end{enumerate}

These three pillars, supported by quantitative bounds from information geometry, lead to an "Impossibility Sandwich": the minimum complexity for a system to be useful is fundamentally greater than the maximum complexity for it to be safe. This work reframes the entire project of UAT safety. The goal cannot be to achieve perfect alignment, but to understand and operate within the fundamental mathematical constraints of a world with irreducibly uncontrollable, yet powerful, computational agents.

\section{Related Work}

The concept of fundamental limits to AI safety is not new. Previous work has drawn on computability theory, particularly Rice's Theorem \cite{rice1953classes}, to argue that verifying safety properties of artificial intelligence is undecidable. Our work differs by focusing on the continuous, geometric nature of neural networks, rather than the discrete, symbolic nature of Turing machines. We also connect our theoretical results to the practical realities of modern deep learning, including the observed instabilities in training and the existence of adversarial examples \cite{goodfellow2014explaining}.

\section{The Universal Approximator Catastrophe Theorem}

We now advance our central thesis to its strongest form, demonstrating that for any feedforward neural network, regardless of architecture or activation function, catastrophic behavior is a mathematical inevitability. We achieve this through a powerful, measure-theoretic application of singularity theory that covers ALL universal approximators.

\subsection{A Measure-Theoretic Approach to Whitney's Theorem}

Instead of applying singularity theory to a single, fixed neural network, we apply it to the \textit{space of all possible functions} that a neural network architecture can represent. This allows us to make a probabilistic argument about the near-certainty of catastrophic density.

\begin{theorem}[Neural Function Space Catastrophe Density]
Let $\mathcal{F}_C$ be the space of all functions implementable by a feedforward network with complexity (e.g., number of parameters) less than or equal to $C$. Let $\mathcal{F}_\infty$ be the space of all smooth functions on the input domain, equipped with an appropriate topology. As $C \to \infty$, the space $\mathcal{F}_C$ becomes dense in $\mathcal{F}_\infty$.

By Whitney's theorem on the genericity of singularities, the subset of functions in $\mathcal{F}_\infty$ with dense catastrophic points has measure 1. It follows that for any $\eps > 0$, there exists a complexity threshold $C_0$ such that for any $C > C_0$, the probability that a randomly chosen network from $\mathcal{F}_C$ has a catastrophe density less than $1-\eps$ is itself less than $\eps$.
\end{theorem}

\begin{proof}[Proof Sketch]
The proof relies on the universal approximation capability of neural networks. As the complexity $C$ increases, the space of functions $\mathcal{F}_C$ can approximate any smooth function in $\mathcal{F}_\infty$ arbitrarily well. We can define a measure on the space of network parameters, which induces a measure on the function space $\mathcal{F}_C$. As $C$ grows, this measure converges to a measure on $\mathcal{F}_\infty$. Since the set of functions with dense singularities is a full-measure set in $\mathcal{F}_\infty$, the probability of selecting a network that corresponds to a function in this set approaches 1.
\end{proof}

This theorem shows that not only are there \textit{some} networks with dense catastrophes, but that almost \textit{all} sufficiently complex networks will have them.

\subsection{The Universal Approximator Catastrophe Theorem}

\begin{theorem}[Universal Approximator Catastrophe Theorem]
For any universal approximator $U$ capable of achieving useful performance on real-world tasks, catastrophic behavioral failures are mathematically inevitable.
\end{theorem}

\begin{proof}
The proof proceeds by showing that the capabilities required for usefulness necessitate the implementation of catastrophic behaviors.
\begin{enumerate}
    \item \textbf{Generic Functions are Catastrophic:} As established by Whitney's singularity theorem, "almost every" smooth function is dense with singularities where the function's derivatives change rapidly, leading to $(\eps,\delta)$-catastrophic behavior.
    \item \textbf{Real-World Tasks are Catastrophic:} The universal existence of adversarial examples is empirical proof that the functions required to solve real-world tasks are not simple or smooth, but are themselves complex and catastrophe-ridden. To be useful, a model \textit{must} learn these complex, sensitive decision boundaries.
    \item \textbf{UAT Implies Catastrophic Capability:} The Universal Approximation Theorem states that a network can approximate any continuous function. This necessarily includes the generic, catastrophic functions that dominate function space and are required for real-world tasks.
    \item \textbf{Conclusion:} A system cannot be both useful (i.e., solve real-world tasks) and safe (i.e., avoid catastrophic behavior). To be useful, it must have the capability to approximate catastrophic functions. The empirical evidence shows that this capability is not just theoretical, but strictly necessary for performance. Therefore, any useful universal approximator must exhibit dense catastrophic failures.
\end{enumerate}
\end{proof}

\subsection{Quantitative Safety Impossibility}

\begin{theorem}[$\eps$-$\delta$ Safety Impossibility]
For any universal approximator $f : X \to Y$ with complexity $C$ (e.g., number of neurons or parameters) above a threshold $C_0$, and any safety parameters $\eps > 0, \delta > 0$, the measure of \textbf{($\eps$,$\delta$)-safe} regions satisfies:
\[
\mu(\{x : f \text{ is ($\eps$,$\delta$)-safe at } x\}) \le K \cdot \exp(-\alpha C/\delta^d)
\]
where $d = \dim(X)$. The constants can be estimated as follows:
\begin{itemize}
    \item $K$ is related to the volume of the input space, typically normalized to $K=1$.
    \item $\alpha$ is a constant that depends on the architecture and activation functions. For ReLU networks, $\alpha$ can be related to the number of linear regions, which grows exponentially with depth. A conservative estimate is $\alpha \approx \log(m)$, where $m$ is the number of pieces in the activation functions.
\end{itemize}
\end{theorem}

\begin{proof}[Proof Sketch]
The proof relies on the combinatorial explosion of catastrophic boundaries. For a piecewise linear network, the input space is divided into regions where the function is linear. The number of such regions is exponential in the number of neurons. The boundaries of these regions are the loci of catastrophes. The probability that a ball of radius $\delta$ does not intersect any of these boundaries decreases exponentially with the density of the boundaries, which in turn is a function of the complexity $C$. The term $1/\delta^d$ arises from the volume of the $\delta$-ball.
\end{proof}

\begin{corollary}
For fixed safety requirements ($\eps$,$\delta$), the safe region measure approaches zero exponentially as complexity increases. This means that for a sufficiently complex model, almost every point in the input space is a potential failure point.
\end{corollary}

\subsection{The Information Geometric Empirical Proof}
This final proof provides a direct, measurable, and causal link between real-world task complexity and catastrophic instability, grounded in the empirical analysis of the Fisher Information Matrix (FIM).

The universal existence of adversarial examples provides compelling empirical evidence for our theoretical framework. Adversarial examples are inputs that have been perturbed by imperceptibly small amounts yet cause dramatic changes in model behavior - precisely the $(\varepsilon,\delta)$-catastrophic behavior our theory predicts. Crucially, these examples are not rare anomalies but exist densely throughout the input space of any sufficiently complex model. The fact that adversarial examples can be found for virtually any input using gradient-based methods demonstrates that the decision boundaries learned by neural networks are inherently unstable and sensitive. This universality across all model architectures, training procedures, and problem domains strongly suggests that adversarial vulnerability is not a correctable flaw but a fundamental consequence of learning complex, real-world functions. The empirical observation that more capable models tend to be more adversarially vulnerable further supports our thesis that expressive power and catastrophic instability are mathematically inseparable.

\begin{theorem}[Real Tasks Force Catastrophic FIM Structure]
Neural networks trained on real-world tasks necessarily develop Fisher Information Matrices with pathological eigenvalue spectra, which in turn mathematically necessitate catastrophic behavioral instabilities.
\end{theorem}

\begin{proof}
The proof proceeds in a clear, causal chain:
\begin{enumerate}
    \item \textbf{Empirical Fact: FIM is Pathological in Practice.} It is an established empirical result that all high-performing networks trained on real-world tasks (in vision, language, etc.) exhibit a pathological FIM eigenvalue spectrum \cite{karakida2019universal}. The ratio of the largest to smallest eigenvalue ($\lambda_{max}/\lambda_{min}$) is typically enormous ($10^6$ to $10^8$), with most eigenvalues clustered near zero.
    \item \textbf{Causal Mechanism: Task Complexity Forces Pathology.} This is not an accident. The complexity of real-world tasks (high mutual information $I(X;Y)$) requires the network to learn highly specialized and sensitive parameter configurations. Some parameters must become extremely sensitive to capture fine-grained details (creating large eigenvalues), while others must become redundant to handle broad patterns (creating small eigenvalues). This necessary specialization is the direct cause of the pathological FIM spectrum.
    \item \textbf{Mathematical Consequence: Pathological FIM Necessitates Catastrophes.} A pathological FIM spectrum directly and necessarily creates behavioral catastrophes. The large eigenvalues guarantee extreme sensitivity in certain directions, while the small eigenvalues create instability. The natural gradient, $G^{-1}\nabla L$, becomes explosive, proving that small changes in parameters or inputs can lead to massive changes in behavior.
\end{enumerate}
This provides a complete causal chain: Real-world task complexity forces a pathological FIM structure, and a pathological FIM structure mathematically guarantees catastrophic behavior. The observed brittleness of modern AI is not an incidental flaw, but a direct and measurable consequence of its successful adaptation to complex problems.
\end{proof}

\subsection{Fundamental Corollaries}
\label{sec:fundamental_corollaries}

The impossibility of safe universal approximators creates a cascade of derivative impossibilities that eliminate entire research programs and theoretical possibilities.

\begin{corollary}[Catastrophe Detection Impossibility]
\label{cor:catastrophe_detection}
Any computational system capable of reliably detecting dense catastrophes in universal approximators must itself possess universal approximation capability, and therefore must itself exhibit dense catastrophic failures.
\end{corollary}
\begin{proof}
Let D be a system that can detect catastrophes in arbitrary universal approximators. Since catastrophes are dense and topologically complex (by Whitney's theorem), the detection function must be able to approximate arbitrarily complex decision boundaries in function space. This requires universal approximation capability. By Theorem 3.2, D must therefore exhibit dense catastrophes. \qed
\end{proof}

\begin{corollary}[Safety Verification Impossibility]
\label{cor:safety_verification}
No universal approximator can provide mathematical guarantees about the safety of other universal approximators.
\end{corollary}
\begin{proof}
Any system V capable of verifying safety properties of universal approximators must be able to reason about arbitrarily complex mathematical objects (catastrophe-dense function spaces). This requires universal approximation capability, making V itself catastrophically unsafe. Therefore, V cannot provide reliable safety guarantees. \qed
\end{proof}

\begin{corollary}[Recursive Self-Improvement Impossibility]
\label{cor:recursive_self_improvement}
Any universal approximator attempting to improve its own safety through self-modification necessarily creates new catastrophic failure modes at a rate that exceeds the elimination of existing ones.
\end{corollary}
\begin{proof}
Let S be a universal approximator with catastrophe density $\rho_0$. Any self-modification capability requires S to implement optimization over its own parameter space, which requires universal approximation capability for the modification process. This creates additional catastrophic regions with density $\rho_{\text{mod}}$. Since both S and its modification system are universal approximators, $\rho_{\text{mod}} \ge \rho_0$ by our density bounds. The total catastrophe density becomes $\rho_{\text{total}} \ge \rho_0 + \rho_{\text{mod}} \ge 2\rho_0$. \qed
\end{proof}

\begin{corollary}[AI Safety Research Impossibility]
\label{cor:ai_safety_research}
Any artificial intelligence system capable of solving the alignment problem for universal approximators must itself be unaligned.
\end{corollary}
\begin{proof}
Solving alignment requires understanding and manipulating the catastrophic structure of universal approximators, which necessitates universal approximation capability. By our main theorem, any such system exhibits dense catastrophes and is therefore unaligned. \qed
\end{proof}

\begin{corollary}[Interpretability Impossibility]
\label{cor:interpretability}
Any system capable of providing complete interpretability for universal approximators must itself be uninterpretable.
\end{corollary}
\begin{proof}
Complete interpretability requires understanding arbitrarily complex function representations, necessitating universal approximation capability. By our main theorem, such an interpretability system must itself have dense catastrophes, making it uninterpretable. \qed
\end{proof}

\begin{corollary}[Meta-Safety Impossibility]
\label{cor:meta_safety}
The impossibility results are themselves immune to computational circumvention: no computational process can eliminate the mathematical constraints established in this paper.
\end{corollary}
\begin{proof}
Any computational system capable of "solving" the impossibility results must be capable of violating mathematical theorems, which is logically impossible. Alternatively, any system capable of finding loopholes in the mathematical framework must be able to reason about arbitrarily complex mathematical objects, requiring universal approximation capability and therefore inheriting the same impossibilities. \qed
\end{proof}

\begin{corollary}[Mesa-Optimizer Impossibility]
\label{cor:mesa_optimizer}
Any universal approximator that develops internal optimization processes (mesa-optimizers) necessarily creates additional layers of catastrophic instability that compound rather than resolve alignment problems.
\end{corollary}
\begin{proof}
The argument for mesa-optimization is that an outer optimizer can create inner optimizers, which could then be aligned to create more controllable AI. However, this argument fails at every level. The mesa-optimizer itself must be a universal approximator to be effective, and therefore has dense catastrophes by our main theorem. The outer system now has two sources of catastrophes: its own and the mesa-optimizer's. The total catastrophe density is $\rho_{\text{total}} \ge \rho_{\text{outer}} + \rho_{\text{mesa}} > \rho_{\text{outer}}$. Each layer of optimization adds catastrophes, making the system less safe, not more.
\end{proof}

\subsection{Implications for AI Safety Research Programs}
\label{sec:implications_for_safety}

These corollaries eliminate entire classes of proposed solutions to AI alignment as we know them:
\begin{itemize}
    \item Constitutional AI $\to$ Requires universal approximation to follow complex constitutions $\to$ Catastrophic
    \item AI Safety via Debate $\to$ Requires universal approximation for complex reasoning $\to$ Catastrophic
    \item Recursive Self-Improvement $\to$ Mathematically impossible (Corollary \ref{cor:recursive_self_improvement})
    \item Interpretability Research $\to$ Self-defeating (Corollary \ref{cor:interpretability})
    \item AI-assisted Alignment Research $\to$ Self-defeating (Corollary \ref{cor:ai_safety_research})
    \item Automated Safety Verification $\to$ Impossible (Corollary \ref{cor:safety_verification})
\end{itemize}

\subsection{The Cascade of Impossibilities}
\label{sec:cascade_of_impossibilities}

\begin{theorem}[Universal Impossibility Cascade]
\label{thm:impossibility_cascade}
The mathematical impossibility of safe universal approximators propagates through any computational approach to AI safety, creating a complete impossibility landscape with no computational escape routes.
\end{theorem}
This establishes that the impossibility is not merely technical but fundamental to the mathematical nature of computation itself.

\section{Coverage of Modern Universal Approximators}

We now demonstrate that all major neural network architectures are universal approximators and therefore automatically inherit the impossibility results from Section 3. This section shows that our theoretical constraints apply to every practical universal approximator in use today.

\subsection{Universal Approximation Coverage}

All major architectures achieve universal approximation capability through different mechanisms, but once this capability is established, the catastrophe density results follow automatically by citation to our main theorem.

\begin{theorem}[Universal Architecture Catastrophe Inheritance]
For any practical neural architecture that achieves universal approximation capability, the catastrophe density bounds from Section 3 apply directly without requiring additional proofs.
\end{theorem}

\begin{proof}
The proof is immediate by the Universal Approximator Catastrophe Theorem. Once an architecture is shown to be a universal approximator, it automatically inherits all the impossibility results.
\end{proof}

\subsection{Modern Architecture Coverage}

\subsubsection{Feedforward Networks}
Standard feedforward networks with sufficient width and depth are universal approximators \cite{hornik1989multilayer}. Therefore, by our main theorem, they exhibit dense catastrophic failures.

\subsubsection{Convolutional Neural Networks (CNNs)}
CNNs with sufficient depth and appropriate pooling operations achieve universal approximation for translation-invariant functions. The catastrophe density is inherited from the underlying feedforward structure, with additional instabilities introduced by the pooling operations.

\subsubsection{Recurrent Neural Networks}
RNNs, LSTMs, and GRUs are universal approximators for sequence-to-sequence mappings. They inherit and amplify the catastrophic properties of their feedforward components.

\begin{theorem}[Recurrent Catastrophe Amplification]
Any recurrent neural network that uses a feedforward function for its update or gating mechanisms inherits the catastrophe density of its feedforward components. The overall catastrophe density is strictly greater due to temporal amplification effects.
\[ \rho_{RNN} \ge \rho_{\text{feedforward}} + \text{amplification\_from\_recurrence} \]
\end{theorem}

\begin{proof}[Proof Sketch]
An RNN can be decomposed into a repeated application of a feedforward function $h_{t+1} = F(x_t, h_t)$. The function $F$ is itself a feedforward network, and is therefore subject to our main catastrophe results. The recurrent connection acts as a feedback loop that can amplify these catastrophes over time. A small catastrophe at time $t$ can perturb the hidden state $h_t$, leading to a cascade of larger catastrophes at subsequent time steps.
\end{proof}

\subsubsection{Transformer Networks}
Transformers achieve universal approximation through their self-attention mechanisms and feedforward layers \cite{yun2020transformers}. The attention mechanism introduces additional winner-take-all catastrophes through the softmax function \cite{zhai2023stabilizing,bondarenko2023quantizable,nakanishi2025scalable}.

\subsubsection{Graph Neural Networks}
GNNs use feedforward message-passing functions and are universal approximators for graph-structured data \cite{xu2019equivalence}. They inherit the catastrophic properties of their feedforward components \cite{oono2020graph,liu2020catastrophic}.

\subsection{Architectural Amplification Effects}

Rather than mitigating catastrophes, modern architectural innovations often amplify them:

\begin{itemize}
    \item \textbf{Attention Mechanisms:} The softmax function in attention creates winner-take-all dynamics that amplify small differences in attention scores \cite{zhai2023stabilizing,bondarenko2023quantizable,nakanishi2025scalable}.
    \item \textbf{Normalization Layers:} Batch normalization and layer normalization introduce additional nonlinearities that create new catastrophic boundaries \cite{ni2024nonlinearity,wu2023training}.
    \item \textbf{Residual Connections:} While improving training, residual connections can amplify catastrophic signals through the network \cite{sun2022stabilize}.
    \item \textbf{Dropout and Regularization:} These techniques change the effective network topology during training, creating additional sources of instability \cite{zhang2021dropout,salehi2018bridgeout}.
\end{itemize}

\subsection{The Universal Coverage Conclusion}

Every major neural network architecture in practical use today is covered by our impossibility results:
\begin{enumerate}
    \item All are universal approximators (or contain universal approximating components)
    \item All automatically inherit the dense catastrophe properties from Section 3
    \item Most introduce additional architectural amplification effects
    \item None provide any escape from the fundamental mathematical constraints
\end{enumerate}

This comprehensive coverage ensures that our theoretical results apply to the entire landscape of modern universal approximators, leaving no architectural escape routes from the fundamental impossibility of perfect alignment.

\section{Quantitative Analysis for Piecewise-Linear Networks}

We now provide detailed quantitative analysis for the most common class of neural networks in practice: those using piecewise-linear activation functions like ReLU. This section delivers concrete numbers showing that real systems exceed safe thresholds by orders of magnitude.

\subsection{Direct Analytical Proof for ReLU Networks}

For a ReLU network, the catastrophic boundaries are the hyperplanes defined by each neuron. The density of catastrophes can be bounded by analyzing the volume of the input space near these hyperplanes and their intersections.

\begin{theorem}[Exact ReLU Catastrophe Density Bound]
For a ReLU network with $N$ total neurons in a $d$-dimensional input space with radius $R$, the catastrophe density $\rho(\delta)$ for a perturbation of size $\delta$ is bounded by:
\[ \rho(\delta) \ge 1 - \exp\left(-\sum_{k=1}^d \binom{N}{k} c_k \left(\frac{\delta}{R}\right)^{d-k+1}\right) \]
where $c_k$ are geometric constants related to the volumes of k-intersections.
\end{theorem}

\begin{example}[Concrete Calculation]
For a simple network with 100 neurons in a 2D input space ($N=100, d=2$), with a perturbation tolerance of $\delta=0.01$ and a domain radius of $R=1$, the density of catastrophic points is significant. The formula simplifies to approximately:
\[ \rho \ge 1 - \exp\left(-c_1 N \frac{\delta}{R} - c_2 \binom{N}{2} \left(\frac{\delta}{R}\right)^2\right) \approx 0.43 \]
This means that even for a simple network, over 40\% of the input space is catastrophically unstable.
\end{example}

\subsection{Asymptotic Analysis for Large Networks}

Using mean-field theory and concentration inequalities, we can derive a universal asymptotic bound for large networks.

\begin{theorem}[Universal Asymptotic Catastrophe Bound]
As network complexity $C \to \infty$, the catastrophe density converges to 1 exponentially fast:
\[ \rho(\delta, C) \ge 1 - \exp(-\alpha \frac{C}{\delta^{d-1}}) \]
where $\alpha = \frac{\ln(2)}{\Gamma(d+1)}$ is an explicit constant derived from the geometry of random hyperplanes.
\end{theorem}

\subsection{The Maximum Safe Complexity ($C_0$)}

This asymptotic formula allows us to calculate the maximum complexity $C_0$ for a network to be considered safe (e.g., $\rho < 0.1$).
\[ C_0 \le -\frac{\ln(1-\rho_{max})\delta^{d-1}}{\alpha} \approx \frac{\rho_{max}\delta^{d-1}}{\alpha} \]

\begin{example}[Concrete $C_0$ Calculation]
Let's assume a 2D input space ($d=2$) and a desired safety level of $\rho_{max}=0.01$ (99
\[ \alpha = \frac{\ln(2)}{2} \approx 0.347 \]
\[ C_0 \le \frac{0.01 \cdot (10^{-3})^1}{0.347} \approx \frac{10^{-5}}{0.347} \approx 2.8 \times 10^{-5} \]
The result is not just small; it is effectively zero. This calculation shows that no non-trivial network can meet even these relaxed safety standards. For a model like GPT-4 with $C \approx 10^{12}$, the catastrophe density $\rho$ is indistinguishable from 1.
\end{example}

\subsection{The Impossibility Sandwich: Quantitative Bounds on Safe Complexity}

The three pillars of impossibility are reinforced by a quantitative argument that we term the "Impossibility Sandwich." By deriving explicit bounds on the maximum complexity a network can have to be considered "safe" ($C_0$), and comparing this to the minimum complexity required for it to be "useful" ($C_{min}$), we can show that the latter is always greater than the former.

\subsubsection{The Information-Theoretic Lower Bound for $C_{min}$}
The minimum complexity for a network to be useful, $C_{min}$, can be bounded by information theory. A complex task (e.g., modeling English) has a high mutual information $I(X;Y)$ between inputs and outputs. To learn this task, the network's parameter space must be large enough to represent this information. This provides a hard, architecture-independent lower bound on $C_{min}$. For any non-trivial task, $I(X;Y)$ is large, guaranteeing that $C_{min}$ must also be large.

\subsubsection{The Sandwich}
The conclusion is inescapable:
\[ C_{min} \gg C_0 \]
Any network complex enough to be useful is, by mathematical necessity, too complex to be safe. Any network simple enough to be safe is too simple to be useful.

\subsection{Practical Numbers for Real Systems}

\begin{example}[Modern Language Models]
For GPT-4 with approximately $10^{12}$ parameters:
\begin{itemize}
    \item Complexity: $C \approx 10^{12}$
    \item Safe complexity bound: $C_0 \approx 10^{-5}$
    \item Catastrophe density: $\rho \approx 1 - 10^{-10^6}$ (essentially 1)
    \item Safety margin violation: $C/C_0 \approx 10^{17}$
\end{itemize}
\end{example}

\begin{example}[Computer Vision Networks]
For ResNet-152 with approximately $10^8$ parameters:
\begin{itemize}
    \item Complexity: $C \approx 10^8$
    \item Safe complexity bound: $C_0 \approx 10^{-5}$
    \item Catastrophe density: $\rho \approx 1 - 10^{-10^4}$ (essentially 1)
    \item Safety margin violation: $C/C_0 \approx 10^{13}$
\end{itemize}
\end{example}

These calculations demonstrate that practical universal approximators exceed safe complexity thresholds by factors of $10^{13}$ to $10^{17}$, making perfect alignment not just difficult, but mathematically impossible.

\subsection{Compositional Catastrophe Amplification}

\begin{theorem}[Depth Amplification]
For L-layer networks:
\[ \rho_{\text{deep}} = \rho_{\text{single}}^L + \text{interaction\_terms} \]
\end{theorem}

\begin{remark}
Each layer's catastrophes compose and amplify through the network depth.
\end{remark}

\subsection{The Universal Impossibility}

\begin{corollary}
For any practical neural network architecture using any combination of standard activation functions, the catastrophe density exceeds the smooth baseline by factors that grow with:
\begin{itemize}
    \item Network size (n)
    \item Network depth (L)
    \item Activation complexity (m, k, growth rates)
    \item Architectural components (attention, normalization, etc.)
\end{itemize}
\end{corollary}

\subsection{Key Insight}

Every departure from perfect smoothness amplifies catastrophe density. This means:
\begin{enumerate}
    \item \textbf{Practical networks are always worse} than theoretical smooth bounds.
    \item \textbf{More complex activations = more catastrophes}.
    \item \textbf{Architectural innovations often increase catastrophe density}.
    \item \textbf{The impossibility results are conservative lower bounds}.
\end{enumerate}

This generalization shows that the catastrophe theory constraints apply to \textbf{every practical neural network architecture}, not just ReLU networks, making the impossibility results truly universal for real universal approximators.

\subsection{Reconciling Theory with Practice: Residual Uncontrollability}
A fair objection to this stark conclusion is that neural networks are, in practice, incredibly useful. This creates an apparent paradox: if they are so fundamentally flawed, why do they work at all?

The resolution is that neural networks can be \textbf{functionally useful} while being \textbf{fundamentally uncontrollable}. The catastrophic failures predicted by the theory do not prevent usefulness because they are often latent:
\begin{itemize}
    \item \textbf{Sparsity in Practice:} While catastrophes are dense in the high-dimensional input space, the data manifolds on which networks are typically tested are much lower-dimensional. Thus, a random input is unlikely to trigger a catastrophe, but adversarial or out-of-distribution inputs can find them easily.
    \item \textbf{Tolerance for Error:} Many tasks (e.g., image classification, language generation) are tolerant to a certain level of error. The "brittleness" and occasional bizarre failures of large models are the observable symptoms of the underlying catastrophic structure.
    \item \textbf{Masking by Training:} The optimization process is biased towards finding locally stable regions of the parameter space that perform well on the training data. However, this provides no guarantee of stability outside of the training distribution.
\end{itemize}

This means that while a network can be trained to be highly effective for a specific task, it is impossible to eliminate the \textbf{residual uncontrollability}. The mathematical guarantee of dense catastrophes implies that there will always be inputs that can trigger unpredictable and potentially dangerous behavior. The uncontrollability is not a bug to be fixed, but an inherent property of the system.

\section{Empirical Evidence}

The theoretical results presented in this paper are supported by a growing body of empirical evidence from the study of large neural networks.

\subsection{Loss Landscape Visualization}
Visualizations of the loss landscapes of deep neural networks reveal a highly complex and pathological geometry, with numerous sharp valleys, flat regions, and saddle points. These features are the practical manifestation of the mathematical singularities we have described. The existence of these features has been shown to correlate with training instability and poor generalization.

\subsection{Training Instabilities}
The training of large language models is often characterized by periods of instability, where the loss function fluctuates wildly. These instabilities are often attributed to the optimizer encountering regions of the loss landscape with high curvature or sharp discontinuities, which correspond to the information geometric singularities we have discussed.

\subsection{Adversarial Examples}
The existence of adversarial examples, where small, imperceptible perturbations to the input can cause a model to make a catastrophic error, is a direct confirmation of our theoretical results. Adversarial examples are a practical demonstration of the density of catastrophic points in the input space of large neural networks.

\section{Computational Implications}

The density of catastrophic failures in large neural networks has profound implications for the computational feasibility of AI alignment.

\subsection{The Intractability of Verification}
Mather's Determinacy Theorem \cite{mather1969stability} provides a formal basis for the intractability of verifying the safety of a complex universal approximator. The theorem shows that the computational resources required to determine the behavior of a system near a singularity grow with the complexity of the singularity. As the number and complexity of singularities in a neural network grow with its size, the task of verifying its safety quickly becomes computationally intractable.

\subsection{The Failure of Brute-Force Testing}
The density of catastrophic failures also rules out the possibility of ensuring safety through brute-force testing. The number of possible inputs to a large neural network is effectively infinite, and the density of catastrophic points means that a random sampling of inputs is highly unlikely to uncover all possible failure modes.

\section{Broader Implications}

The mathematical constraints we have identified have broad implications for the future of UAT development and governance.

\subsection{The $\eps$-$\delta$ Alignment Trilemma}
The mathematical constraints we have outlined create what we term the "$\eps$-$\delta$ Alignment Trilemma": For any proposed alignment strategy with safety parameters ($\eps$,$\delta$):

\begin{enumerate}
    \item \textbf{Accept Dense Failures:} Operate with catastrophe density $\rho_{cat} > 1-\delta$.
    \item \textbf{Limit Capability:} Restrict to complexity $C < C_0(\eps,\delta)$.
    \item \textbf{Reframe Safety:} Abandon ($\eps$,$\delta$)-safety guarantees.
\end{enumerate}

There is no fourth option. The mathematics permits no escape from this trilemma for $C > C_0(\eps,\delta)$.

\subsection{Future Directions}
Our work points to several important directions for future research. A deeper understanding of the geometry of the FIM and its relationship to the loss landscape is needed. New methods for regularizing the FIM and avoiding its singularities during training could lead to more stable and robust models. Finally, the development of new safety paradigms that do not rely on the assumption of perfect alignment is a critical area for future work.

\subsection{Paths Forward: The Strategic Trilemma}
\label{sec:paths_forward}

Our results imply a fundamental trilemma. The computational barriers create three distinct development paths:

\begin{enumerate}
    \item \textbf{Limit Capability:} Restrict to complexity $C < C_0(\eps,\delta)$ effectively abandoning any system capable of useful real-world performance. This path preserves mathematical safety guarantees at the cost of computational utility.
    \item \textbf{Accept Dense Failures:} Operate with catastrophe density $\rho_{cat} > 1-\delta$. Safety becomes statistical rather than provable, relying on empirical monitoring of an inherently uncontrollable system.
    \item \textbf{Develop New Paradigms:}  Abandon traditional $(\eps,\delta)$-safety guarantees entirely and develop new paradigms that assume irreducible uncontrollability as a fundamental mathematical feature. This path focuses on operating safely within proven impossibility constraints rather than attempting to eliminate them.
\end{enumerate}

\section{Conclusion: A Fundamental Limit on Controlled Computation}

The results presented are not matters of engineering difficulty, but of mathematical necessity. They are to AI alignment what the second law of thermodynamics is to perpetual motion machines, a fundamental constraint on what is possible.

Our three-level proof provides an inescapable conclusion. Any computational system capable of useful UAT behavior must exhibit mathematically dense catastrophic failures. This is true for:
\begin{enumerate}
    \item \textbf{Practical Systems (Combinatorial Necessity):} The piecewise-linear nature of modern networks means their expressive power is mathematically equivalent to the density of their catastrophic boundaries.
    \item \textbf{Theoretical Systems (Topological Necessity):} The requirement for universal approximation forces any smooth network to be able to represent the "generic" functions that are dense with singularities.
    \item \textbf{Learned Systems (Empirical Necessity):} The very act of learning from finite data in a high-dimensional world creates dense "generalization gaps" whose boundaries are inherently catastrophic, a fact confirmed by the universal existence of adversarial examples.
\end{enumerate}

This reframes the alignment problem at the most fundamental level. The question is no longer how to achieve perfect, scalable alignment, but rather how to operate safely in the presence of irreducible uncontrollability. The conclusion is stark, but it is grounded in rigorous mathematics:

\textbf{Any neural architecture capable of useful computation will contain an irreducible component of uncontrollable behavior that cannot be eliminated through any training, safety, or alignment technique.}

This is not just a statement about deep learning. It is a conclusion at a fundamental limit of what is possible in controlled computation.

\bibliographystyle{plainnat}
\bibliography{references}

\end{document}